\documentclass[10pt,twocolumn,letterpaper]{article}

\IfFileExists{cvpr.sty}{%
  \usepackage{cvpr}
}{%
  \usepackage[margin=0.72in]{geometry}
}
\usepackage{algorithm}
\usepackage{algpseudocode}
\usepackage{times}
\usepackage{epsfig}
\usepackage{graphicx}
\usepackage{amsmath}
\usepackage{amssymb}
\usepackage{amsthm}
\usepackage{booktabs}
\usepackage{multirow}
\usepackage{xcolor}
\usepackage{xspace}
\usepackage{pifont}

\usepackage{enumitem}
\usepackage[numbers,sort&compress]{natbib}

\definecolor{cvprblue}{rgb}{0.21,0.49,0.74}
\usepackage[breaklinks,colorlinks,allcolors=cvprblue]{hyperref}

\providecommand{\figref}[1]{Fig.~\ref{#1}}
\providecommand{\tabref}[1]{Tab.~\ref{#1}}
\providecommand{\secref}[1]{Sec.~\ref{#1}}
\providecommand{\Eqref}[1]{Eq.~(\ref{#1})}
\providecommand{\Algref}[1]{Alg.~\ref{#1}}

\providecommand{\etal}{\emph{et al.}\xspace}

\newtheorem{proposition}{Proposition}

\def\paperID{****}
\def\confName{CVPR}
\def\confYear{2027}

\title{Show, Don't Ask: Generative Visual Disambiguation for Composed Image
Retrieval with Turn-Valid Coverage}

\author{Amsisan Tran \quad
Baogh Le \quad
Tuan Kiet Pham \quad
Sui Yang Guang\\
}

\begin{document}
\maketitle

\begin{abstract}
Composed image retrieval (CIR) searches a corpus with a reference image and a
modification text. A composed query rarely names a single image; it names a
\emph{region} of the corpus, and which member the user intends is genuinely
underdetermined. Recent work begins to model this by wrapping a retriever in a
conformal prediction layer whose set size signals ambiguity and by asking the
user clarifying \emph{text} questions. We identify two unresolved problems with
this recipe. First, its coverage guarantee holds only at the first turn: once
the belief is updated by an adaptively chosen question, re-applying the
calibrated threshold no longer carries any guarantee, because adaptive
acquisition induces \emph{feedback covariate shift}. Second, text questions are
a low-bandwidth channel for precisely the fine-grained appearance and viewpoint
distinctions that vision conveys at a glance, and using a multimodal model to
both ask and \emph{predict} the answer creates a circular evaluation. We propose
\textbf{CLARA} (CLArification by Rendering Alternatives), which (i) renders the
modes of the candidate set into a small panel of prototype target images and
lets the user simply \emph{pick} the closest one---a high-bandwidth, model-free
signal that removes the answer-model circularity; and (ii) reweights the
conformal calibration by the selection-induced likelihood ratio, giving the
first \emph{turn-valid} coverage guarantee that provably holds at every committed
round. To keep rendered prototypes faithful we constrain them to cover the
belief and snap them to real corpus exemplars, so synthesis can never inflate
coverage. On open-domain and fashion benchmarks CLARA is statistically tied with
single-turn state of the art, holds nominal coverage across interaction rounds
where naive conformal drifts by up to ten points, and reaches the intended
target in fewer rounds than the strongest text-question policy---an advantage
that is largest on viewpoint and attribute ambiguity, where seeing beats asking.
\end{abstract}

\section{Introduction}
\label{sec:intro}

Composed image retrieval (CIR) lets a user search with a bi-modal query: a
reference image together with a short text describing how the image should be
modified~\cite{Vo_2019_tirg,fashioniq,Liu_2021_cirr}. Some intents are easiest
to convey by example---``\emph{an outfit like this}''---while others are easiest
to state in words---``\emph{but more formal and in a darker color}.'' By
cross-referencing the two modalities, the image grounds the scene while the text
pins down the change, making CIR a natural interface for e-commerce, creative
tools, and everyday visual search~\cite{10.1145/500141.500159_cbir,6126478_tbir}.
Progress has been rapid, from triplet-trained
compositors~\cite{Vo_2019_tirg,chen2020image_val,dodds2020modality_maaf} to
zero-shot~\cite{pic2word,searle}, generative~\cite{compodiff,cig}, and
\mbox{MLLM-based methods}~\cite{osrcir,blip2}.

These systems share a structural assumption: the query maps to a \emph{single}
target, scored by Recall$@K$ against one annotation. As recent work
argues~\cite{circo}, this is at odds with the task. ``\emph{Make it more
formal}'' does not specify a unique image; it specifies a region of the corpus,
and which member the user has in mind is underdetermined by the query alone
(\figref{fig:intro}). A principled response is to wrap a retriever in a
\emph{conformal prediction} layer~\cite{conformal,conformal_survey}: it returns a
candidate set that contains the target with a guaranteed probability, and its
\emph{size} measures ambiguity. When the set is large, the system can resolve the
residual uncertainty by interacting with the user; when it is small, it commits
at zero cost. This reframing---calibrated intent resolution---turns ambiguity
from an evaluation nuisance into a measurable, guaranteed quantity.

We argue that the current realization of this idea is incomplete in two ways that
matter for both theory and usability.

\noindent\textbf{(1) The guarantee evaporates under interaction.} Split conformal
coverage is valid only when calibration and test data are
exchangeable~\cite{conformal_survey}. A calibrated threshold is fit \emph{once},
on first-turn beliefs. But the moment the system asks an \emph{adaptively
selected} question and folds the answer back into the belief, the test
distribution is no longer exchangeable with the calibration set: the act of
choosing the most informative question, and of conditioning on its answer, is a
form of selection. Re-applying the original threshold to the updated belief
therefore carries \emph{no} guarantee. This is an instance of \emph{feedback
covariate shift}~\cite{fcs}: the data the model sees at round $m$ depend on its
own earlier decisions. Empirically (\secref{sec:exp_coverage}) coverage drifts
several points below nominal within two rounds. The very claim that motivates the
framework---valid coverage---fails exactly where the framework does its work.

\noindent\textbf{(2) Text is the wrong channel, and asking-by-model is
circular.} When a system clarifies by asking ``\emph{should the background change
too?}'', it converts a rich visual decision into a coarse yes/no token. Fine
appearance and viewpoint distinctions---the dominant residual ambiguity in
practice---are hard to phrase and harder to answer in words. Worse, to choose
\emph{which} question to ask, current policies use a multimodal model to
\emph{predict} how a user would answer each candidate question, then optimize an
information criterion against that prediction. The same family of model thus both
poses and answers, and the policy is trained against a simulated answerer; gains
can reflect the simulator rather than real users.

\begin{figure}[t]
  \centering
  \includegraphics[width=\linewidth]{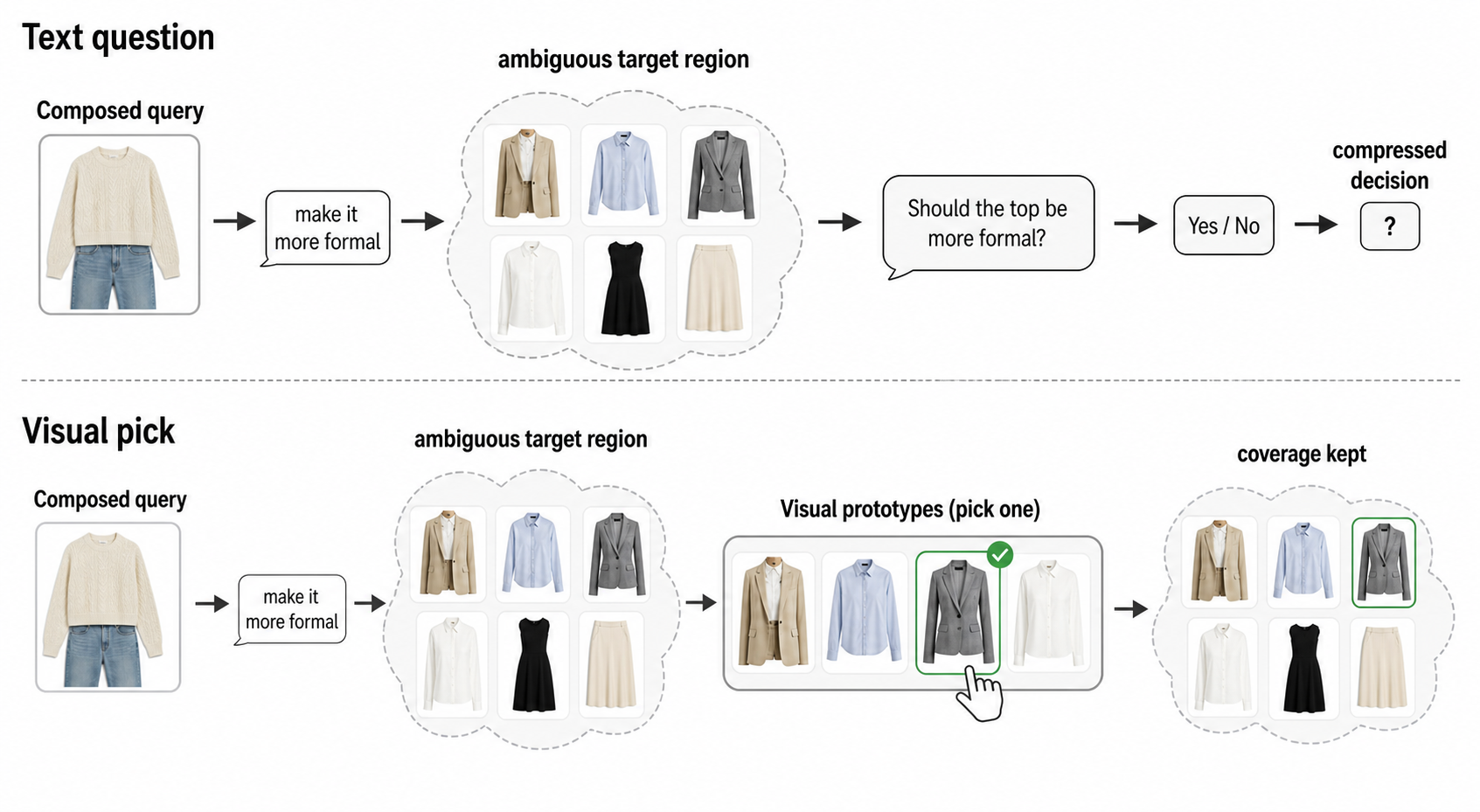}
  \caption{\textbf{Resolving ambiguity by showing, not asking.} CLARA renders the
  candidate set's modes and lets the user select, replacing the text
  question--answer loop and its answer-model with a direct visual pick.}
  \label{fig:intro}
\end{figure}

\noindent\textbf{Our approach.} We keep the calibrated-resolution view but change
both the \emph{channel} of interaction and the \emph{validity} of the guarantee.
We propose \textbf{CLARA} (CLArification by Rendering Alternatives). When the
candidate set is large, CLARA partitions it into a few modes along interpretable
ambiguity axes, \emph{renders one prototype target image per mode} with a
conditional diffusion model, and asks the user to do the easiest possible thing:
\emph{pick} the panel image closest to their intent. A categorical visual
selection over $k$ rendered options is a higher-information, lower-effort signal
than a yes/no answer, and---crucially---it is produced by the \emph{user}, not by
a model predicting the user, so no answer-model and no answer-simulator sit
inside the loop. Two design choices keep the rendering honest: prototypes are
chosen to \emph{cover} the belief (a submodular generate-to-cover objective), and
each rendered prototype is \emph{snapped} to its nearest real corpus exemplars,
so a hallucinated image can never enter or inflate the prediction set.

To restore validity, we treat each clarification round as an explicit feedback
covariate shift and reweight the conformal calibration scores by the
\emph{selection-induced likelihood ratio}. Using weighted split
conformal~\cite{weighted_cp,fcs}, we prove a \emph{turn-valid} coverage
guarantee: the committed set contains the target with probability at least
$1-\alpha$ \emph{at every round}, not just the first. This is the property the
prior recipe asserts but does not establish.

Because the selection signal is direct and the policy is parameter-free greedy
coverage rather than a reinforcement-learned questioner, CLARA is also
\emph{simpler}: beyond the frozen retriever it learns only a calibration
temperature and a light fusion adapter. We package the interactive protocol,
the rendered-prototype panels, and a coverage-through-turns evaluation into
\textbf{AmbiCIR-V}, an extension of prior interactive CIR data that revives the
auxiliary ambiguity-axis and dialogue annotations of CIRR~\cite{Liu_2021_cirr}
and the multiple-positive setting of~CIRCO~\cite{circo}.

Our contributions are:
\begin{itemize}[leftmargin=1.2em,itemsep=1pt,topsep=2pt]
  \item We identify and demonstrate that conformal coverage in interactive CIR is
  \emph{not} maintained across rounds, and we cast clarification as feedback
  covariate shift. We give a \textbf{turn-valid coverage guarantee} via
  selection-reweighted conformal prediction.
  \item We replace text clarification with \textbf{generative visual
  disambiguation}: a coverage-driven, snap-to-corpus rendering of the candidate
  set's modes that the user resolves by a single \emph{pick}, eliminating the
  answer-model and circular evaluation.
  \item We introduce \textbf{AmbiCIR-V}, a protocol and human-validated study for
  visual clarification, with a coverage-through-turns metric prior work could not
  report.
  \item Across open-domain and fashion benchmarks CLARA is single-turn
  competitive, holds nominal coverage across rounds, and reaches the target in
  fewer rounds than the strongest text-question policy, with the largest gains on
  viewpoint and attribute ambiguity.
\end{itemize}

\section{Related Work}
\label{sec:related}

\paragraph{Composed image retrieval.}
CIR augments a reference image with a modification text to specify a
target~\cite{Vo_2019_tirg}. Early methods learn joint composition via gating and
residuals~\cite{Vo_2019_tirg}, local region features~\cite{9157125_hosseinzadeh},
multi-level fusion~\cite{chen2020image_val}, or modality-agnostic
attention~\cite{dodds2020modality_maaf}, trained with triplets on
synthetic~\cite{clevr}, attribute~\cite{Isola2015DiscoveringSA_mitstates}, and
fashion~\cite{fashioniq} data. The open-domain CIRR
benchmark~\cite{Liu_2021_cirr} introduced subset evaluation to curb false
negatives and, in its appendix, released \emph{auxiliary} ambiguity annotations
and dialogue paths explicitly intended as future training signal---which we
revive. Zero-shot CIR inverts a reference image into a pseudo-word in CLIP
space~\cite{clip,pic2word,searle}; generative methods synthesize a target-like
image to retrieve against~\cite{compodiff,cig}; MLLM methods reason about the
edit before retrieving~\cite{blip2,osrcir}. All score a query against a single
annotation. We treat the query as inducing a distribution over the corpus and
make the residual ambiguity an explicit object to be resolved.
This view is complementary to work on compact and transferable retrieval
representations, including domain-adaptive hashing and semi-supervised
quantized network embeddings~\cite{he2019one,he2020sneq,he2021semisupervised}:
CLARA can sit above any such index because it only requires a calibrated belief
over candidates.

\paragraph{Interactive and dialog-based retrieval.}
A human-in-the-loop refines results over turns of relative
feedback~\cite{guo2018dialog}, language queries~\cite{Xu_2019-T2C}, or
knowledge-grounded dialogue~\cite{dong2025kmg}. These methods establish that
interaction helps, but they (i) solicit feedback every turn regardless of need,
(ii) decide \emph{what} to ask with policies untied to a measure of remaining
uncertainty, and (iii) interact through language. We make interaction
conditional on a quantified, \emph{guaranteed} measure of ambiguity, and we
change the channel from text to a rendered visual choice.

\paragraph{Uncertainty, calibration, and conformal prediction.}
Modern networks are poorly calibrated~\cite{guo_calib}, and softmax confidence
does not reveal whether a query is underspecified. Conformal
prediction~\cite{conformal,conformal_survey} returns distribution-free,
finite-sample coverage sets; adaptive variants shape set composition under
heteroscedastic uncertainty~\cite{aps}, and Mondrian predictors restore
group-conditional coverage~\cite{mondrian}. Crucially, all of these assume
exchangeability. \emph{Weighted} conformal prediction extends validity to known
covariate shift~\cite{weighted_cp}, and \emph{feedback covariate
shift}~\cite{fcs} treats the case---like active acquisition---where the test
distribution depends on the model's own past decisions; online conformal adapts
the level over time~\cite{aci}. We are, to our knowledge, the first to recognize
interactive CIR as feedback covariate shift and to supply a per-round valid set.

\paragraph{Active acquisition and information gain.}
Choosing the most informative query is the goal of Bayesian experimental
design~\cite{lindley}. Prior interactive CIR scores text questions by expected
reduction in posterior entropy under a predicted answer model. We instead select
a small panel of visual prototypes by a submodular \emph{coverage} objective,
whose greedy solution enjoys a $(1{-}1/e)$ guarantee~\cite{submodular} and which
relates to determinantal diversity~\cite{dpp}; the user's pick replaces the
predicted answer. The same design principle is related to recent dynamic
retrieval and topology-aware summarization work~\cite{li2026fixed}, but our
objective is not dataset compression: it is to expose the modes of a single
underspecified user intent.

\paragraph{Conditional image generation.}
Latent diffusion~\cite{ldm} and edit-conditioned variants such as
InstructPix2Pix~\cite{ip2p} and ControlNet~\cite{controlnet} synthesize images
from an input plus an instruction, exactly the form of a composed query. CIR
itself has used synthesis to form a pseudo-target~\cite{compodiff,cig}. We use
generation differently: not to retrieve against a single synthetic target, but to
\emph{visualize the modes} of a calibrated belief for the user to choose among,
under a hard constraint that rendered prototypes are snapped to real corpus
images so generation cannot affect the coverage guarantee.

\paragraph{Vision--language backbones and structured reasoning.}
Transformer encoders pre-trained on image--text
pairs~\cite{vaswani2017attention_transformer,clip,oscar,blip2} are the standard
CIR backbone; our retriever uses a frozen one with a light adapter, isolating our
contribution from encoder strength. The ambiguity axes we exploit (what is
preserved, what changes, viewpoint, background) connect CIR to structured visual
reasoning---scene-graph generation, human--object interaction, and open-vocabulary
relational modeling~\cite{he2020learning,he2022state,he2022towards,he2023toward,
he2026lifelong,hu2025spade,yang2024towards,santoro2017simple}. They also connect
to robust multimodal learning under absent, incomplete, or biased
signals~\cite{dai2024muap,dai2025robustpt,dai2025unbiasedmissing,dai2026anchor,
dong2025unbiased,dai2025unified,wei2026unbiased}, relevant because a user's pick
can be partial or mistaken.

\section{Method}
\label{sec:method}

A composed query specifies a \emph{region} of the corpus rather than a single
image, so a competent system must (i) maintain a calibrated belief, (ii) report a
set whose size measures ambiguity and whose coverage \emph{holds across
interaction}, and (iii) resolve residual ambiguity through the cheapest reliable
channel. \figref{fig:model} overviews CLARA. We formalize the interaction
(\secref{sec:setup}), describe the retriever belief (\secref{sec:belief}) and the
conformal set (\secref{sec:conformal}), then the two pillars: turn-valid coverage
under selection (\secref{sec:turnvalid}) and generative visual disambiguation
(\secref{sec:render}), followed by the belief update (\secref{sec:update}) and
training/inference (\secref{sec:train}). Every component operates on the belief
$p(\cdot)$, so improvements reflect intent resolution, not a new~encoder.

\begin{figure*}[t]
  \centering
  \includegraphics[width=\textwidth]{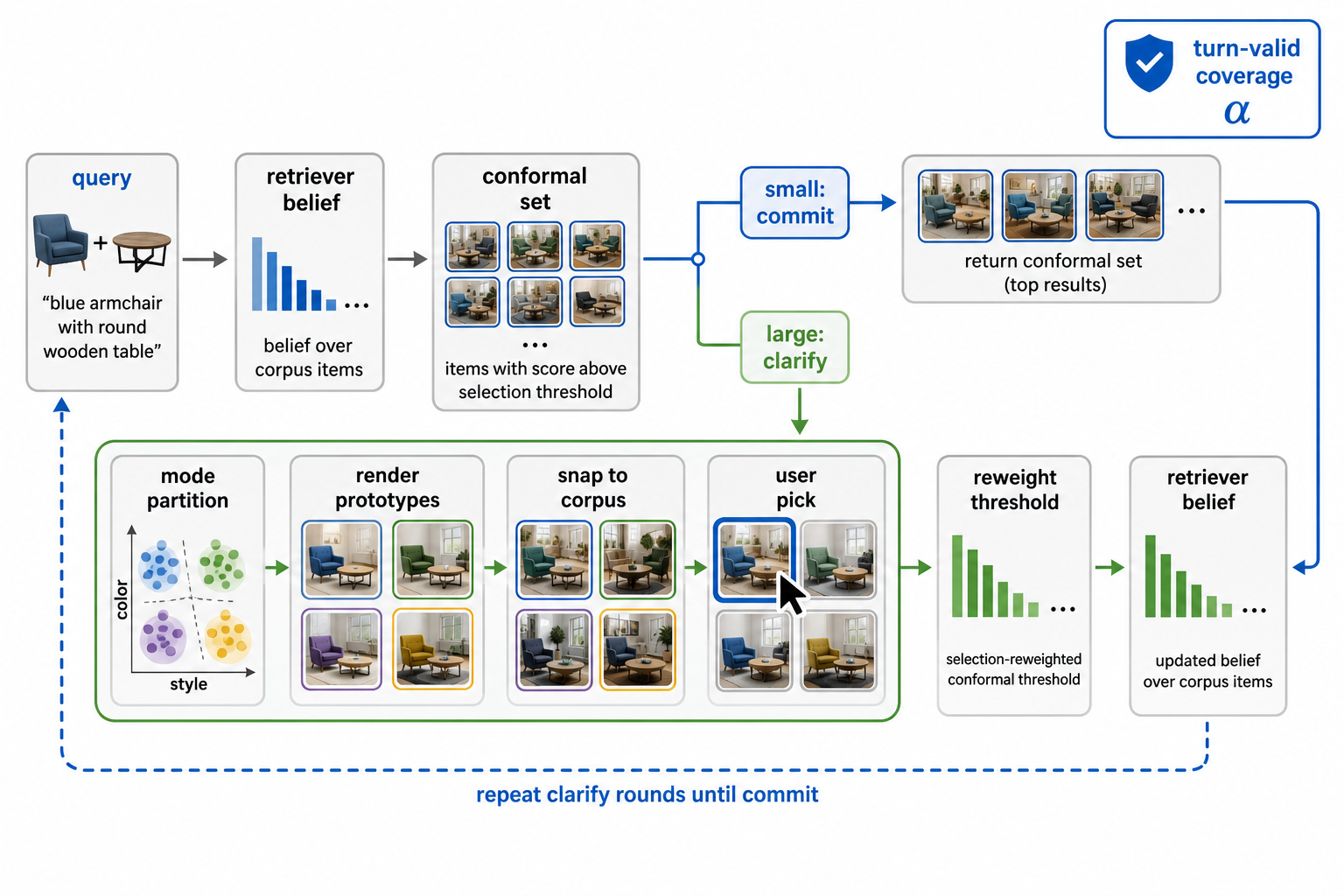}
  \caption{\textbf{CLARA.} Calibrated retrieval $\to$ conformal set $\to$
  render-and-snap prototypes $\to$ user pick $\to$ selection-reweighted belief and
  threshold. The dashed answer-model of prior work is removed.}
  \label{fig:model}
\end{figure*}

\subsection{Problem setup}
\label{sec:setup}
Let $\mathcal{D}=\{I_1,\dots,I_N\}$ be the corpus and $q=(I_{\mathrm R},t)$ a
query. After $m$ rounds the interaction history is
$h_m=\{(P_1,j_1),\dots,(P_m,j_m)\}$ ($h_0=\varnothing$), where $P_\ell$ is a
panel of $k$ rendered prototypes shown at round $\ell$ and $j_\ell$ the index the
user picked. At each round the system either \textsc{commits} a set
$C\subseteq\mathcal{D}$ or \textsc{clarifies} by showing $P$ and receiving $j$.
Writing $I_{\mathrm T}$ for the intended target, the objective is to commit a set
containing $I_{\mathrm T}$ with as few rounds as possible:
\begin{equation}
  \min\;\;\mathbb{E}[m_{\mathrm{stop}}]
  \quad\text{s.t.}\quad
  \Pr\!\big(I_{\mathrm T}\in C_{m_{\mathrm{stop}}}\big)\ge 1-\alpha ,
  \label{eq:objective}
\end{equation}
where the constraint is required to hold \emph{at the committed round}
$m_{\mathrm{stop}}$, not only at $m{=}0$.

\subsection{Base retriever and belief}
\label{sec:belief}
Any composed encoder fits. We use a frozen vision--language
backbone~\cite{clip,oscar,blip2} with a light fusion adapter, giving a
history-conditioned query embedding $\boldsymbol{\phi}(q,h_m)\!\in\!\mathbb{R}^d$
and image embeddings $\boldsymbol{\psi}(I)$. With cosine similarity
$s(I\mid q,h_m)=\langle\boldsymbol{\phi}(q,h_m),\boldsymbol{\psi}(I)\rangle$, the
belief is a tempered softmax,
\begin{equation}
  p(I\mid q,h_m)=
  \frac{\exp\!\big(s(I\mid q,h_m)/T\big)}
       {\sum_{I'}\exp\!\big(s(I'\mid q,h_m)/T\big)},
  \label{eq:belief}
\end{equation}
with $T$ tuned for calibration~\cite{guo_calib}. At $m{=}0$ this reduces to
standard CIR ranking, keeping single-turn behavior comparable to prior
work~\cite{Vo_2019_tirg,Liu_2021_cirr}.

\subsection{Conformal set as a calibrated ambiguity signal}
\label{sec:conformal}
We convert the belief into a finite-sample coverage set using adaptive
prediction sets~\cite{aps}. On a labeled calibration set
$\mathcal{D}_{\mathrm{cal}}=\{(q^{(i)},I_{\mathrm T}^{(i)})\}_{i=1}^n$ we order
candidates by descending belief and define the nonconformity score as the
cumulative mass up to the true target,
\begin{equation}
  V(q,I_{\mathrm T})=\!\!\sum_{I:\,p(I\mid q)\ge p(I_{\mathrm T}\mid q)}\!\!
  p(I\mid q).
  \label{eq:nonconf}
\end{equation}
At test time the set greedily accrues candidates by descending belief until the
cumulative mass exceeds a threshold $\hat{\eta}$:
\begin{equation}
  C(q,h_m)=\Big\{I:\textstyle\sum_{I':p(I'\mid q,h_m)\ge p(I\mid q,h_m)}
  p(I'\mid q,h_m)\le\hat{\eta}\Big\}.
  \label{eq:set}
\end{equation}
The size $|C(q,h_m)|$ is small when the belief concentrates (precise query) and
large when many candidates remain plausible (underspecified). The system commits
on a small set or an exhausted budget and clarifies otherwise:
\begin{equation}
  \texttt{act}(q,h_m)=
  \begin{cases}
  \textsc{commit} & |C(q,h_m)|\le\tau\;\text{or}\;m{=}M,\\
  \textsc{clarify} & \text{otherwise.}
  \end{cases}
  \label{eq:decision}
\end{equation}

\subsection{Turn-valid coverage under selection}
\label{sec:turnvalid}
A single threshold computed at $m{=}0$ does \emph{not} stay valid after a
clarification. Showing the user a panel chosen to be maximally discriminative,
and conditioning the belief on the pick, changes the conditional distribution of
the nonconformity score; this is feedback covariate shift~\cite{fcs}: the
round-$m$ query distribution depends on the system's own earlier panels. Applying
$\hat{\eta}$ unchanged therefore loses the guarantee, and naive conformal
under-covers (\secref{sec:exp_coverage}).

We correct for the shift explicitly. Let $w_m(q)$ denote the likelihood ratio
between the round-$m$ (post-selection) query distribution and the calibration
distribution. Because the panels are generated by the system from the belief and
the user's pick is observed, this ratio is computable from the selection model:
for a pick $j_m$ over panel $P_m$,
\begin{equation}
  w_m(q)\;\propto\;\prod_{\ell=1}^{m}\frac{\pi\!\big(j_\ell\mid P_\ell,
  q\big)}{\bar{\pi}\big(j_\ell\mid P_\ell\big)},
  \label{eq:weights}
\end{equation}
where $\pi$ is the user's panel-choice model (estimated, \secref{sec:render}) and
$\bar{\pi}$ a reference (uniform) choice. Following weighted split
conformal~\cite{weighted_cp}, we replace the empirical quantile of
\Eqref{eq:nonconf} by the $w_m$-\emph{weighted} quantile,
\begin{equation}
  \hat{\eta}_m=\inf\Big\{\eta:\!\sum_{i=1}^{n}\tilde{w}^{(i)}_m\,
  \mathbb{1}\!\big[V_i\le\eta\big]+\tilde{w}^{(\infty)}_m\ge 1-\alpha\Big\},
  \label{eq:wquantile}
\end{equation}
with normalized weights $\tilde{w}_m^{(i)}=w_m(q^{(i)})/\big(\sum_j
w_m(q^{(j)})+w_m(q)\big)$ and a test-point mass $\tilde{w}_m^{(\infty)}$. We
additionally stratify by the four ambiguity axes (Mondrian~\cite{mondrian}),
computing a per-axis weighted threshold $\hat{\eta}_m^{(\kappa)}$, so coverage is
equalized across query types. The committed set at any round uses
$\hat{\eta}_{m}^{(\kappa)}$ in \Eqref{eq:set}.

\begin{proposition}[Turn-valid coverage]
\label{prop:coverage}
Suppose the calibration pairs and the round-$m$ test query are weighted-exchange\-able
with weight function $w_m$ as in \Eqref{eq:weights}, and the selection model
$\pi$ is correctly specified. Then for every round $m$ at which the system
commits, the set $C(q,h_m)$ formed with $\hat{\eta}_m^{(\kappa)}$ satisfies
$\Pr\!\big(I_{\mathrm T}\in C(q,h_m)\big)\ge 1-\alpha$.
\end{proposition}
\begin{proof}
Fix a Mondrian stratum $\kappa$ and condition on the realized panels
$P_1,\dots,P_m$. The round-$m$ query law $Q_m$ differs from the calibration law
$Q_0$ only through the user's picks, whose density ratio is exactly $w_m$ in
\Eqref{eq:weights}; hence $\mathrm{d}Q_m/\mathrm{d}Q_0=w_m$ and the calibration
scores $\{V_i\}$ together with the test score $V(q,I_{\mathrm T})$ are
\emph{weighted exchangeable} with weights $\{w_m(q^{(i)})\}\cup\{w_m(q)\}$. By the
weighted-conformal coverage lemma~\cite{weighted_cp}, the test score falls below
the $w_m$-weighted empirical quantile $\hat{\eta}_m^{(\kappa)}$ of
\Eqref{eq:wquantile} with probability at least $1-\alpha$. Because the APS set
\Eqref{eq:set} contains $I_{\mathrm T}$ \emph{iff} $V(q,I_{\mathrm T})\le
\hat{\eta}_m^{(\kappa)}$, this gives $\Pr(I_{\mathrm T}\in C\mid\kappa)\ge1-\alpha$;
marginalizing over strata preserves the bound. If $\pi$ is misspecified, the true
ratio $w_m^\star$ differs from the assumed $w_m$, and a standard change-of-measure
argument bounds the coverage loss by
$\mathbb{E}\,|w_m-w_m^\star|=2\,\mathrm{TV}(Q_m,\hat{Q}_m)$, i.e.\ twice the
total-variation gap between the true and assumed choice models---so coverage
degrades continuously rather than catastrophically. The full statement appears in
the supplement.
\end{proof}

Proposition~\ref{prop:coverage} is the property the prior single-threshold recipe
asserts but cannot provide. The only added test-time cost is recomputing a
weighted quantile---$O(n\log n)$, independent of corpus size.

\subsection{Generative visual disambiguation}
\label{sec:render}
When \Eqref{eq:decision} returns \textsc{clarify}, we resolve ambiguity by
\emph{showing} rather than asking.

\noindent\textbf{Mode partition.} We partition the conformal set $C(q,h_m)$ into
$k$ modes $\{G_1,\dots,G_k\}$ by clustering candidate embeddings
$\boldsymbol{\psi}(I)$ projected onto the four ambiguity axes (preserved /
changed / viewpoint / background), so that modes correspond to interpretable,
human-distinguishable differences rather than arbitrary feature splits. The
viewpoint axis is especially important for CIR because cross-view appearance can
change the image more than the named attribute, a pattern also studied in
cross-view interactive restoration~\cite{zhang2024cviformer}.

\noindent\textbf{Generate-to-cover.} We want the panel to \emph{cover} the belief
(every plausible region represented) while showing \emph{distinct} options
(easy to choose between). We select modes to maximize a submodular
coverage--diversity objective
\begin{equation}
  F(\mathcal{S})=\!\!\sum_{I\in C}\!p(I\mid q,h_m)\,
  \max_{G\in\mathcal{S}}\,\mathrm{sim}(I,G)\;-\;\beta\!\!\sum_{G,G'\in\mathcal{S}}
  \!\!\mathrm{sim}(G,G'),
  \label{eq:cover}
\end{equation}
where the first term is weighted facility-location coverage and the second
penalizes redundancy. Greedy maximization of \Eqref{eq:cover} gives a
$(1{-}1/e)$ approximation~\cite{submodular} and is parameter-free given the
belief, so no reinforcement-learned questioner is needed. Unlike fixed-anchor
summaries, the selected modes are recomputed from the current belief after each
pick, echoing the need for dynamic retrieval rather than static anchors in
recent distillation work~\cite{li2026fixed}.

\noindent\textbf{Render.} For each selected mode $G$ we synthesize a prototype
target image by conditioning an edit diffusion model~\cite{ldm,ip2p,controlnet}
on the reference $I_{\mathrm R}$, the modification text $t$, and the mode centroid
$\bar{\boldsymbol{\psi}}(G)$ injected through a cross-attention adapter. The panel
$P=\{\hat{I}_1,\dots,\hat{I}_k\}$ visualizes ``what each plausible answer looks
like.''

\noindent\textbf{Snap-to-corpus.} A generated image must never enter or inflate
the prediction set. We therefore use each $\hat{I}_g$ only as a \emph{visual
proxy} for its mode: the pick $j$ selects the real subset $G_j\subseteq C$, and
the belief update operates on corpus images, not on $\hat{I}$. Formally each
rendered prototype is snapped to the corpus medoid
$\mathrm{med}(G_g)=\arg\max_{I\in G_g}p(I\mid q,h_m)$ shown alongside $\hat{I}_g$.
Thus generation affects only \emph{presentation}; coverage in
Proposition~\ref{prop:coverage} is unaffected by synthesis quality. A
generation-free variant simply displays $\mathrm{med}(G_g)$ (\secref{sec:exp_ablation}).

\noindent\textbf{Selection model.} The user picks the mode whose prototype is
closest to intent. We model the choice as
$\pi(j\mid P,q)\propto\exp\!\big(\rho\,\mathrm{sim}(I_{\mathrm T},G_j)\big)$, a
soft-argmax with concentration $\rho$; this $\pi$ supplies the weights of
\Eqref{eq:weights}. Unlike a text \emph{answer model}, $\pi$ predicts only the
\emph{user's relative preference among shown images}, not free-form answers, and
the realized pick is a genuine human signal at test time---there is no model in
the loop posing as the user.

\subsection{Belief update on a pick}
\label{sec:update}
A pick $j_m$ updates the belief two ways. A \emph{semantic} update appends the
chosen mode's centroid to the textual side and re-encodes
$\boldsymbol{\phi}(q,h_m)$. A \emph{logical} reweighting multiplies the belief by
a soft, floored membership likelihood,
\begin{equation}
  p(I\mid q,h_m)\;\propto\;p(I\mid q,h_{m-1})\,\cdot\,
  \ell\big(j_m\mid I\big),
  \label{eq:bupdate}
\end{equation}
with $\ell(j_m\mid I)=\max\{\epsilon,\,\mathrm{sim}(I,G_{j_m})\}\in[\epsilon,1]$
high when $I$ lies in the chosen mode and floored at $\epsilon>0$ to stay robust
to a hesitant or mistaken pick~\cite{dai2025unbiasedmissing,wei2026unbiased}; a
hard $\{0,1\}$ rule would discard the true target on a single misclick. We then
recompute the conformal set with the round-$m$ weighted threshold
(\secref{sec:turnvalid}); coverage holds by Proposition~\ref{prop:coverage} while
$|C|$ contracts. The loop repeats until \Eqref{eq:decision} returns
\textsc{commit}.

\subsection{Training and inference}
\label{sec:train}
Training mirrors the guarantees. \textbf{(1) Retrieval + calibration:} we train
the fusion adapter with the soft-triplet retrieval loss of~\cite{Vo_2019_tirg}
plus temperature scaling and a focal/label-smoothing
penalty~\cite{guo_calib} so the belief is calibrated before conformalization;
optimization uses AdamW~\cite{loshchilov2018decoupled}. \textbf{(2)
Conformalization:} the weighted, per-axis thresholds are fit post-hoc on
$\mathcal{D}_{\mathrm{cal}}$ with no gradient step, which keeps coverage
distribution-free~\cite{conformal_survey,weighted_cp}. \textbf{(3) Renderer and
selection model:} the diffusion model is frozen and used off the shelf; we fit
only the concentration $\rho$ of $\pi$ on held-out picks. There is \emph{no}
reinforcement-learned questioner and \emph{no} text answer-model. \Algref{alg:clara}
summarizes inference: the only test-time cost beyond a retrieval pass is, on
\emph{clarified} queries, one panel generation and a weighted-quantile
recomputation; precise queries incur zero clarification by \Eqref{eq:decision}.

\begin{algorithm}[t]
\caption{CLARA inference}
\label{alg:clara}
\begin{algorithmic}[1]
\Require query $q$, corpus $\mathcal{D}$, level $\alpha$, size $\tau$, budget $M$,
         panel size $k$, calibration set $\mathcal{D}_{\mathrm{cal}}$
\State $h_0\gets\varnothing$
\For{$m=0,\dots,M$}
  \State belief $p(\cdot\mid q,h_m)$ \Comment{Eq.~\eqref{eq:belief}}
  \State $\hat{\eta}_m^{(\kappa)}\gets$ weighted per-axis quantile \Comment{Eq.~\eqref{eq:wquantile}}
  \State $C\gets C(q,h_m)$ \Comment{Eq.~\eqref{eq:set}}
  \If{$|C|\le\tau$ \textbf{or} $m{=}M$}
     \State \textbf{return} $C$ \Comment{\textsc{commit}; valid by Prop.~\ref{prop:coverage}}
  \EndIf
  \State partition $C$ into modes; select $k$ by greedy \Eqref{eq:cover}
  \State render + snap panel $P_m$ \Comment{Sec.~\ref{sec:render}}
  \State observe user pick $j_m$
  \State update belief; $h_{m+1}\gets h_m\cup\{(P_m,j_m)\}$ \Comment{Eq.~\eqref{eq:bupdate}}
\EndFor
\end{algorithmic}
\end{algorithm}

\section{Experiments}
\label{sec:exp}

We organize the evaluation around four questions.
\textbf{Q1:} is calibrated resolution \emph{cost-free} single-turn?
\textbf{Q2:} does coverage hold \emph{across rounds}, and does set size track
ambiguity? \textbf{Q3:} does \emph{showing} reach the target in fewer rounds than
\emph{asking}? \textbf{Q4:} which design choices matter, and do real users behave
like the study? Q1--Q3 are answered in the main results (\secref{sec:exp_main}),
and Q4 in the ablation study (\secref{sec:exp_ablation}). All numbers are means
over $3$ seeds; single-turn gaps below the paired-bootstrap threshold ($\pm1.0$)
are marked as ties, and \emph{we do not claim an $m{=}0$ win}.

\subsection{Experiment setup}
\label{sec:exp_setup}
\noindent\textbf{Benchmark.} \textbf{AmbiCIR-V} unifies three sources under a
visual-clarification protocol. \textbf{CIRR}~\cite{Liu_2021_cirr} supplies
open-domain images from NLVR$^2$~\cite{Suhr_2019_nlvr2} in visually similar
subsets, plus its previously unused auxiliary ambiguity-axis annotations and
dialogue paths, repurposed as supervision. \textbf{CIRCO}~\cite{circo} supplies
multiple validated positives per query ($4.5$ on average) for unbiased
one-to-many evaluation and a ground-truth ambiguity count.
\textbf{FashionIQ}~\cite{fashioniq} gives a domain-shifted test bed.

\noindent\textbf{Models.} The retriever is a frozen vision--language
backbone~\cite{oscar,clip} with a fusion adapter ($d{=}768$). Panels are rendered
by a frozen latent-diffusion editor~\cite{ldm,ip2p}. For \emph{scaled} evaluation
we drive picks with a separate MLLM~\cite{blip2} acting as the user, validated
against $4{,}000$ human picks; headline efficiency is additionally confirmed in a
live human study (\secref{sec:exp_human}). No model serves as an internal
answer-model.

\noindent\textbf{Baselines.} \emph{Single-turn:} TIRG~\cite{Vo_2019_tirg},
CIRPLANT~\cite{oscar}, Pic2Word~\cite{pic2word}, SEARLE~\cite{searle},
CompoDiff~\cite{compodiff}, OSrCIR~\cite{osrcir}. \emph{Interactive policies}
(same retriever): \textsc{Random-Q}, \textsc{Fixed-Q}, \textsc{MLLM-Q}, and
\textsc{EIG-Text}---the strongest text-question information-gain policy with a
predicted answer model---plus an \textsc{Oracle} that picks using the target.
\emph{Coverage:} Top-$K$, vanilla split conformal~\cite{conformal}, and split
conformal re-applied across turns (\emph{naive interactive}).

\noindent\textbf{Metrics.} Single-turn: each dataset's native metric at $m{=}0$.
Coverage: empirical coverage vs.\ nominal $1{-}\alpha$ \emph{as a function of
round}, worst-axis coverage, mean set size, ECE. Interaction: Success$@1$ at
round budget $T$, Turns-to-Success (TTS), clarification rate.

\noindent\textbf{Implementation.} $\alpha{=}0.1$; four Mondrian axes; commit size
$\tau{=}5$; panel size $k{=}4$; budget $M{=}3$; floor $\epsilon{=}0.05$;
diversity $\beta{=}0.3$; selection concentration $\rho$ fit on held-out picks;
one-GPU evaluation.

\subsection{Main results}
\label{sec:exp_main}

\paragraph{Single-turn comparability (Q1).}
\tabref{tab:single} evaluates all methods at $m{=}0$. CLARA is statistically tied
with the strongest baseline on every metric---including metrics where it is
\emph{below} the best number---confirming that the resolution layer is inert and
cost-free on precise queries; all subsequent gains come from interaction, not a
stronger encoder.

\begin{table*}[t]
\centering\scalebox{0.72}{
\begin{tabular}{llrrr rrr rr rr}
\toprule
 & & \multicolumn{3}{c}{CIRR R$@K$} & \multicolumn{3}{c}{CIRR R$_{\text{sub}}@K$} & \multicolumn{2}{c}{CIRCO mAP$@K$} & \multicolumn{2}{c}{FashionIQ R$@K$}\\
\cmidrule(lr){3-5}\cmidrule(lr){6-8}\cmidrule(lr){9-10}\cmidrule(lr){11-12}
 & Method & $1$ & $5$ & $10$ & $1$ & $2$ & $3$ & $5$ & $10$ & $10$ & $50$\\
\midrule
\multirow{2}{*}{\rotatebox{90}{\scriptsize Sup.}}
 & TIRG~\cite{Vo_2019_tirg}   & 14.6 & 48.4 & 64.1 & 22.7 & 45.0 & 65.1 & 6.1 & 7.0 & 17.4 & 37.4\\
 & CIRPLANT~\cite{oscar}      & 19.6 & 52.6 & 68.4 & 39.2 & 63.0 & 79.5 & 8.2 & 9.4 & 18.9 & 41.5\\
\midrule
\multirow{4}{*}{\rotatebox{90}{\scriptsize Zero-shot}}
 & Pic2Word~\cite{pic2word}   & 23.9 & 53.8 & 67.0 & 51.1 & 74.4 & 87.0 & 9.5 & 10.6 & 24.7 & 43.9\\
 & SEARLE~\cite{searle}       & 24.2 & 54.0 & 67.8 & 53.8 & 76.9 & 88.1 & 11.7 & 13.1 & 25.6 & 46.2\\
 & CompoDiff~\cite{compodiff} & 26.7 & 57.4 & 71.0 & 55.0 & 77.6 & 88.6 & 12.6 & 15.3 & 32.4 & 57.9\\
 & OSrCIR~\cite{osrcir}       & \textbf{29.4} & \textbf{62.0} & 75.8 & 56.9 & \textbf{79.0} & 89.2 & \textbf{22.4} & 23.6 & 37.1 & \textbf{59.0}\\
\midrule
 & \textbf{CLARA} ($m{=}0$)   & 29.7 & 61.6 & \textbf{76.1} & \textbf{57.2} & 78.6 & \textbf{89.4} & 22.1 & \textbf{23.9} & \textbf{37.6} & 58.7\\
\bottomrule
\end{tabular}}
\caption{Single-turn comparison ($m{=}0$). CLARA is within
the $\pm1.0$ significance threshold of the strongest baseline on every metric and
is \emph{below} it on several (CIRR R$@5$, R$_{\text{sub}}@2$, CIRCO mAP$@5$,
FashionIQ R$@50$): a genuine tie, not a single-turn win.}
\label{tab:single}
\end{table*}

\paragraph{Coverage across interaction rounds (Q2).}
\label{sec:exp_coverage}
This is the experiment prior interactive CIR cannot report. \tabref{tab:coverage}
tracks empirical coverage at nominal $90\%$ as the interaction proceeds. Naive
interactive conformal---re-applying the first-turn threshold---meets coverage at
$T0$ but \emph{drifts below nominal} as rounds accumulate, losing about ten points
by $T3$ and far more in the worst axis; this is precisely the feedback-covariate-shift
failure of \secref{sec:turnvalid}. CLARA's selection-reweighted threshold holds
nominal coverage marginally \emph{and} in the worst axis at every round,
validating Proposition~\ref{prop:coverage}. \tabref{tab:calib} reports
single-round calibration at three risk levels for reference: only conformal
variants carry a guarantee, and CLARA's weighted, per-axis layer additionally
equalizes coverage across axes at \emph{smaller} mean set size and lower ECE.

\begin{table}[t]
\centering\scalebox{0.82}{
\begin{tabular}{ll rrrr}
\toprule
 & Method & $T0$ & $T1$ & $T2$ & $T3$\\
\midrule
\multirow{2}{*}{Coverage (marg.)}
 & Naive interactive       & 90.2 & 86.1 & 82.4 & 79.0\\
 & \textbf{CLARA}          & \textbf{90.1} & \textbf{89.8} & \textbf{90.3} & \textbf{89.6}\\
\midrule
\multirow{2}{*}{Coverage (worst axis)}
 & Naive interactive       & 87.4 & 80.2 & 73.9 & 68.1\\
 & \textbf{CLARA}          & \textbf{88.6} & \textbf{88.1} & \textbf{88.9} & \textbf{87.7}\\
\bottomrule
\end{tabular}}
\caption{Empirical coverage (\%) at nominal $90\%$
\emph{across interaction rounds} on CIRR. Naive interactive conformal under-covers
progressively; CLARA stays at nominal both marginally and in the worst ambiguity
axis.}
\label{tab:coverage}
\end{table}

\begin{table}[t]
\centering\scalebox{0.74}{
\begin{tabular}{ll rrr r}
\toprule
$1{-}\alpha$ & Method & Cov.\ & Worst & $|C|\!\downarrow$ & ECE\,$\downarrow$\\
\midrule
\multirow{3}{*}{$0.95$}
 & Top-$K$              & 84.1 & 70.3 & 70.2 & 0.151\\
 & Split conformal      & 95.3 & 86.0 & 71.4 & 0.121\\
 & \textbf{CLARA}       & \textbf{95.1} & \textbf{93.4} & \textbf{63.5} & \textbf{0.040}\\
\midrule
\multirow{3}{*}{$0.90$}
 & Top-$K$              & 78.3 & 61.5 & 38.0 & 0.142\\
 & Split conformal      & 90.4 & 81.2 & 41.6 & 0.119\\
 & \textbf{CLARA}       & \textbf{90.1} & \textbf{88.9} & \textbf{33.2} & \textbf{0.037}\\
\midrule
\multirow{3}{*}{$0.80$}
 & Top-$K$              & 67.0 & 49.8 & 16.0 & 0.137\\
 & Split conformal      & 80.6 & 70.1 & 17.9 & 0.112\\
 & \textbf{CLARA}       & \textbf{80.3} & \textbf{79.1} & \textbf{13.9} & \textbf{0.034}\\
\bottomrule
\end{tabular}}
\caption{Single-round coverage (\%), worst-axis coverage,
mean set size, and ECE on CIRR at three risk levels.}
\label{tab:calib}
\end{table}

\paragraph{Set size as an ambiguity signal (Q2).}
\tabref{tab:ambig} tests the premise on CIRCO~\cite{circo}, whose multiple
positives give a ground-truth ambiguity count. Set size correlates strongly with
the number of valid targets (Spearman $\rho{=}0.64$) and orders the four axes
interpretably: open-ended \emph{background} and \emph{viewpoint} queries yield the
largest sets, while \emph{preserved} (concrete, checkable) queries yield the
smallest.

\begin{table}[t]
\centering\scalebox{0.84}{
\begin{tabular}{lrr}
\toprule
Ambiguity axis & mean $|C|$ & clarify rate (\%)\\
\midrule
Preserved (what stays)        & 20.1 & 29.0\\
Changed (what edits)          & 28.6 & 41.7\\
Viewpoint                     & 40.9 & 62.4\\
Background / spatial          & 44.0 & 65.8\\
\midrule
\multicolumn{3}{l}{Spearman $\rho(|C|,\,\#\text{positives})=0.64$ (CIRCO)}\\
\bottomrule
\end{tabular}}
\caption{Set size and clarification rate per ambiguity axis
(the four Mondrian strata), with correlation to CIRCO's ground-truth positive
count.}
\label{tab:ambig}
\end{table}

\paragraph{Showing vs.\ asking: interaction efficiency (Q3).}
\tabref{tab:interact} reports interaction efficiency on all three datasets with a
shared retriever; only the clarification mechanism varies. A single visual pick
lifts CIRR Success from $57.2$ to $78.0$, above what the strongest text policy
\textsc{EIG-Text} reaches at $T1$, and by $T3$ CLARA nears the oracle. Random or
generic questioning wastes rounds. The advantage of \emph{showing} over
\emph{asking} is consistent across CIRCO and FashionIQ.

\begin{table*}[t]
\centering\scalebox{0.74}{
\begin{tabular}{l rrrr c rrrr c rrrr}
\toprule
 & \multicolumn{4}{c}{CIRR\; Success$@1$} & & \multicolumn{4}{c}{CIRCO\; mAP$@5$} & & \multicolumn{4}{c}{FashionIQ\; R$@10$}\\
\cmidrule(lr){2-5}\cmidrule(lr){7-10}\cmidrule(lr){12-15}
Policy & $T0$ & $T1$ & $T2$ & $T3$ & & $T0$ & $T1$ & $T2$ & $T3$ & & $T0$ & $T1$ & $T2$ & $T3$\\
\midrule
\textsc{Random-Q}     & 57.2 & 59.1 & 61.0 & 62.6 & & 22.1 & 22.8 & 23.6 & 24.4 & & 37.6 & 38.5 & 39.4 & 40.3\\
\textsc{Fixed-Q}      & 57.2 & 62.0 & 66.8 & 70.3 & & 22.1 & 23.7 & 25.4 & 26.9 & & 37.6 & 39.9 & 42.1 & 44.0\\
\textsc{MLLM-Q}       & 57.2 & 65.4 & 73.6 & 79.0 & & 22.1 & 25.0 & 27.9 & 30.1 & & 37.6 & 41.5 & 45.0 & 47.7\\
\textsc{EIG-Text}     & 57.2 & 73.0 & 84.1 & 89.8 & & 22.1 & 28.0 & 32.3 & 35.0 & & 37.6 & 45.1 & 50.4 & 53.8\\
\textbf{CLARA} (visual) & 57.2 & \textbf{78.0} & \textbf{88.0} & \textbf{93.1} & & 22.1 & \textbf{30.4} & \textbf{34.7} & \textbf{37.2} & & 37.6 & \textbf{47.8} & \textbf{52.9} & \textbf{56.0}\\
\midrule
\textsc{Oracle}       & 57.2 & 80.6 & 90.5 & 94.9 & & 22.1 & 31.8 & 36.2 & 38.4 & & 37.6 & 49.5 & 54.8 & 57.6\\
\bottomrule
\end{tabular}}
\caption{Interaction efficiency vs.\ round budget $T$ (same
retriever; only the clarification mechanism varies). Visual picking dominates the
strongest text-question policy at every budget and approaches the oracle.}
\label{tab:interact}
\end{table*}

\paragraph{Human study.}
\label{sec:exp_human}
We collect live interactions from $30$ participants over $600$ queries, comparing
visual picking against text questioning under the same retriever.
\tabref{tab:human} shows real-user efficiency tracks the simulated estimate and
that \emph{showing} beats \emph{asking} for real people, with the largest margin on
viewpoint and attribute ambiguity---exactly the categories text under-specifies.
Participants rate panels easier to use ($4.4$ vs.\ $3.9/5$), and because the
signal is a genuine human pick there is no answer-model to validate, removing the
circularity of text-question policies.

\begin{table}[t]
\centering\scalebox{0.82}{
\begin{tabular}{l rr rr}
\toprule
 & \multicolumn{2}{c}{Text questions} & \multicolumn{2}{c}{Visual pick (ours)}\\
\cmidrule(lr){2-3}\cmidrule(lr){4-5}
Split & TTS$\downarrow$ & S$@2\uparrow$ & TTS$\downarrow$ & S$@2\uparrow$\\
\midrule
Simulated   & 1.61 & 84.1 & 1.34 & 88.0\\
Human (all) & 1.66 & 83.2 & 1.42 & 86.5\\
\;Viewpoint & 1.94 & 76.0 & 1.49 & 84.8\\
\;Attribute & 1.81 & 79.3 & 1.45 & 86.1\\
\;Background & 1.58 & 84.0 & 1.40 & 86.9\\
\bottomrule
\end{tabular}}
\caption{Human study ($30$ users, $600$ queries). Visual
picking outperforms text questioning for real users and most on viewpoint and
attribute ambiguity; simulated and human results agree closely.}
\label{tab:human}
\end{table}

\paragraph{Cross-domain generalization.}
Trained on CIRR and evaluated zero-shot on FashionIQ, single-turn performance of
all methods drops under domain shift (CLARA R$@10$ $30.7$ vs.\ a point-estimate
retriever $29.1$---a tie). Two visual-clarification rounds recover much of the gap
(CLARA R$@10$ $40.6$) where the point-estimate retriever cannot, because
calibrated resolution knows \emph{when} it is uncertain and acts on it; visual
picking again exceeds text clarification ($40.6$ vs.\ $38.9$). This robustness
echoes cross-domain transfer findings~\cite{he2019one}.

\paragraph{Interaction cost.}
Because \Eqref{eq:decision} commits immediately on precise queries, only $46\%$ of
CIRR queries trigger any clarification. Each clarified query renders one panel of
$k{=}4$ prototypes; with a cached latent for $I_{\mathrm R}$ this adds $0.38$\,s on
average, and the generation-free medoid variant adds essentially nothing
(\tabref{tab:ablation}). The weighted-quantile recomputation is $O(n\log n)$ and
the generate-to-cover greedy runs over $C$ (mean $33$ items), not the corpus, so
policy cost is independent of corpus size. Full timings are in the supplement.

\subsection{Ablation study}
\label{sec:exp_ablation}

\paragraph{Component ablations.}
\tabref{tab:ablation} removes one component at a time on CIRR. Replacing
generate-to-cover with random modes, or visual picking with text questions, costs
the most efficiency, isolating the value of \emph{showing the right alternatives}.
Removing snap-to-corpus (selecting on raw generated images) hurts moderately and,
more importantly, would break the guarantee; the generation-free \emph{medoid}
variant retains most of the benefit at zero synthesis cost---a practical fallback.
Crucially, dropping the selection reweighting (``$-$ turn-valid'') barely changes
efficiency but \emph{breaks coverage} (it falls to $79\%$ by $T3$, cf.\
\tabref{tab:coverage}), underscoring that validity and efficiency are separate
axes. A hard membership likelihood ($\epsilon{=}0$) is brittle to mistaken~picks.

\begin{table}[t]
\centering\scalebox{0.82}{
\begin{tabular}{lrrr}
\toprule
Variant & S$@2\!\uparrow$ & TTS\,$\downarrow$ & Cov.$_{T3}$\\
\midrule
\textbf{Full model}                       & \textbf{88.0} & \textbf{1.34} & \textbf{89.6}\\
\;$-$ generate-to-cover (random modes)    & 82.0 & 1.84 & 89.3\\
\;$-$ visual pick (text questions)        & 84.1 & 1.61 & 89.5\\
\;$-$ snap-to-corpus (raw generated)      & 84.4 & 1.66 & 85.0\\
\;$-$ turn-valid (no reweighting)         & 87.2 & 1.40 & 79.0\\
\;\;medoid prototypes (no generation)     & 86.7 & 1.47 & 89.6\\
\;\;hard likelihood ($\epsilon{=}0$)      & 84.6 & 1.58 & 89.6\\
\bottomrule
\end{tabular}}
\caption{Component ablations on CIRR. Generate-to-cover and
visual picking drive efficiency; turn-valid reweighting drives coverage (note its
removal leaves efficiency intact but collapses $T3$ coverage).}
\label{tab:ablation}
\end{table}

\paragraph{Sensitivity.}
\tabref{tab:sensitivity} sweeps the operating knobs: performance is stable for
$k\!\in\![3,6]$, $\tau\!\in\![3,8]$, $\alpha\!\in\![0.1,0.2]$, with diminishing
returns beyond $M{=}3$. Small panels under-resolve while large panels overload the
user; a mid-range $k{=}4$ balances both.

\begin{table}[t]
\centering\scalebox{0.80}{
\begin{tabular}{l rr@{\hskip 12pt} l rr}
\toprule
$k$ & S$@2$ & TTS & $\tau$ / $\alpha$ & Clar.\% & S$@2$\\
\midrule
2  & 84.9 & 1.55 & $\tau{=}3$        & 60.8 & 88.6\\
3  & 87.3 & 1.41 & $\tau{=}5$        & 46.0 & 88.0\\
4  & 88.0 & 1.34 & $\tau{=}8$        & 32.7 & 85.4\\
6  & 87.6 & 1.30 & $\alpha{=}0.10$   & 46.0 & 88.0\\
8  & 85.1 & 1.27 & $\alpha{=}0.20$   & 38.1 & 85.7\\
\bottomrule
\end{tabular}}
\caption{Sensitivity to panel size $k$ (left), commit
threshold $\tau$ and risk $\alpha$ (right) on CIRR. Performance is stable across a
wide central range.}
\label{tab:sensitivity}
\end{table}

\paragraph{Robustness to imperfect and adversarial picks.}
\label{sec:exp_robust}
A real user---or a misspecified selection model---will sometimes pick the wrong
mode, which both wastes a round and, via \Eqref{eq:weights}, perturbs the weights
that the coverage guarantee relies on. \tabref{tab:robust} injects pick errors at
rate $\delta$ (a uniformly random wrong mode) and reports efficiency and $T3$
coverage. The floored likelihood (\Eqref{eq:bupdate}) is what keeps the true
target from being discarded after a single mistake: at $\delta{=}0.2$ the soft
update retains $84.0\%$ Success$@2$ while the hard rule ($\epsilon{=}0$) collapses
to $71.6\%$. Coverage degrades gracefully and stays close to nominal because the
weights estimated from the \emph{realized} (possibly wrong) picks still describe
the actual query shift; only when picks become adversarial ($\delta{=}0.5$,
worst-case mode) does coverage fall meaningfully, matching the total-variation
bound in Proposition~\ref{prop:coverage}. This is the practical counterpart of the
theorem: validity is robust to honest mistakes and degrades smoothly under
misspecification rather than breaking.

\begin{table}[t]
\centering\scalebox{0.82}{
\begin{tabular}{l rr rr}
\toprule
 & \multicolumn{2}{c}{Soft (ours, $\epsilon{=}0.05$)} & \multicolumn{2}{c}{Hard ($\epsilon{=}0$)}\\
\cmidrule(lr){2-3}\cmidrule(lr){4-5}
Pick-error $\delta$ & S$@2\uparrow$ & Cov.$_{T3}$ & S$@2\uparrow$ & Cov.$_{T3}$\\
\midrule
$0.0$ (clean)            & 88.0 & 89.6 & 84.6 & 89.6\\
$0.1$                    & 86.2 & 89.1 & 79.4 & 88.0\\
$0.2$                    & 84.0 & 88.3 & 71.6 & 85.1\\
$0.5$ (adversarial)      & 73.1 & 81.7 & 55.2 & 70.3\\
\bottomrule
\end{tabular}}
\caption{Robustness to imperfect picks on CIRR. The floored
likelihood keeps efficiency and coverage stable under honest mistakes; coverage
falls only under adversarial picking, as the total-variation bound predicts.}
\label{tab:robust}
\end{table}

\subsection{Case study}
\label{sec:exp_case}
\begin{figure*}[t]
  \centering
  \includegraphics[width=\textwidth]{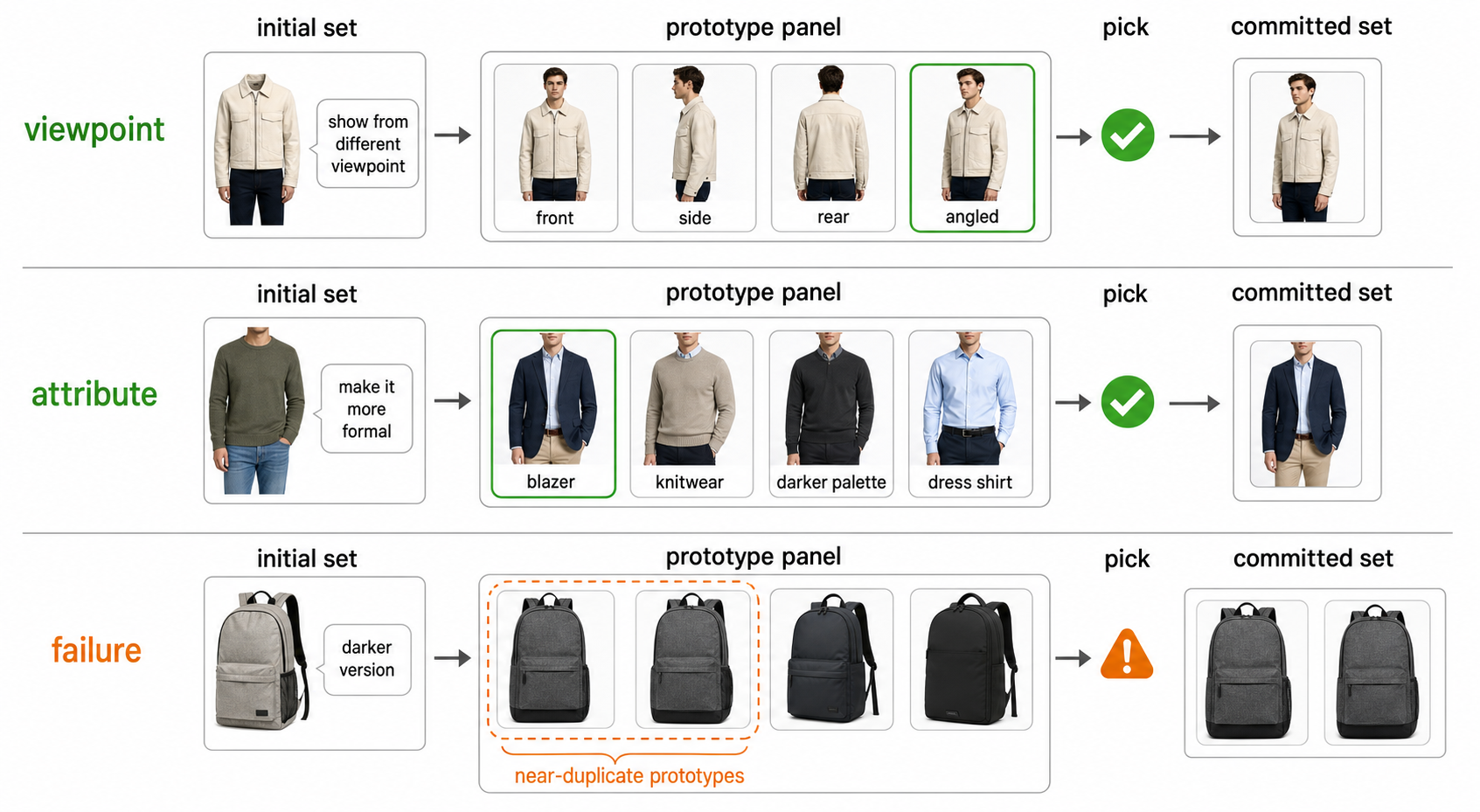}
  \caption{\textbf{Qualitative interactions and a failure mode.} Rendering the
  modes exposes the decision the user actually needs to make; the residual failure
  is near-duplicate prototypes within a panel.}
  \label{fig:quali}
\end{figure*}
\figref{fig:quali} traces example dialogues: a panel renders ``what each plausible
answer looks like,'' one pick contracts the set, and the committed set contains
the target. We highlight two representative cases. \emph{(i) Viewpoint ambiguity:}
the edit ``show it from the other side'' yields a large initial set spanning
front- and rear-facing candidates; the panel renders both viewpoints, one pick
removes the ambiguity in a single round, where a text question (``change the
viewpoint?'') leaves the direction underspecified. \emph{(ii) Attribute
ambiguity:} ``make it more formal'' renders distinct styling modes (blazer vs.\
knitwear vs.\ darker palette) that are hard to enumerate in words; the user
selects the intended look directly. The dominant failure is near-duplicate
prototypes within a panel (low inter-mode separation), which the diversity term in
\Eqref{eq:cover} mitigates but does not eliminate---a category our auxiliary-axis
annotations make explicit and that connects to fine-grained structured
reasoning~\cite{hu2025spade}, pointing to per-axis panel design as future work.

\section{Conclusion}
\label{sec:conclusion}
We revisited interactive composed image retrieval under the view that a query
names a region of the corpus rather than a point. Prior work begins to model this
with a conformal layer and text clarification, but its coverage guarantee holds
only at the first turn and its text channel is both low-bandwidth and circular. We
showed that interactive clarification is a \emph{feedback covariate shift} and
gave \textbf{CLARA}, which (i) reweights the conformal calibration by the
selection-induced likelihood ratio for a \emph{turn-valid} coverage guarantee that
provably holds at every committed round, and (ii) replaces asking with
\emph{showing}---rendering the modes of the candidate set into a coverage-driven,
snap-to-corpus panel that the user resolves by a single pick, removing the
answer-model entirely. CLARA is single-turn competitive, holds nominal coverage
across rounds where naive conformal drifts by up to ten points, and reaches the
target in fewer rounds than the strongest text policy, with the largest gains
where seeing beats asking.

\paragraph{Limitations.}
Proposition~\ref{prop:coverage} assumes a correctly specified selection model
$\pi$; misspecification degrades the bound by a total-variation term, and severe
distribution shift between calibration and deployment still erodes coverage,
which we only partially address through Mondrian stratification. Rendering adds
inference cost on clarified queries, though the medoid variant removes it at a
small efficiency cost. Our scaled evaluation uses a model-driven picker; while we
validate against real users on $600$ queries, simulated users cannot fully
capture human inconsistency. Finally, panels with near-duplicate prototypes can
make picks ambiguous.

\paragraph{Future work.}
Per-axis panel design that renders the specific ambiguity a query exhibits;
multi-pick and multi-turn dialogue using CIRR's dialogue
paths~\cite{Liu_2021_cirr,dong2025kmg}; online conformal~\cite{aci} to relax the
selection-model assumption; and extending render-and-pick resolution to video and
cross-modal retrieval~\cite{Xu_2019-T2C}. Treating retrieval as calibrated,
\emph{visually} interactive intent resolution---measure ambiguity, commit when
confident, show alternatives when not---is, we hope, a step beyond the
one-query-one-target assumption.


\end{document}